\documentclass[conference]{IEEEtran}
\IEEEoverridecommandlockouts
\usepackage{cite}
\usepackage{amsmath,amssymb,amsfonts}
\usepackage{algorithmic}
\usepackage[ruled,vlined]{algorithm2e}
\usepackage{graphicx}
\usepackage{textcomp}
\usepackage{xcolor}
\usepackage{amsthm}
\newtheorem{theorem}{Theorem}

\usepackage{booktabs}
\def\BibTeX{{\rm B\kern-.05em{\sc i\kern-.025em b}\kern-.08em
    T\kern-.1667em\lower.7ex\hbox{E}\kern-.125emX}}
\begin{document}

\title{KC-Agent: A Dual-Process Cognitive Architecture for Efficient ML Model Improvement\\
}

\author{\IEEEauthorblockN{
Gusseppe Bravo-Rocca\IEEEauthorrefmark{1}\IEEEauthorrefmark{4},
Jordi Guitart\IEEEauthorrefmark{1}\IEEEauthorrefmark{2},
Ajay Dholakia\IEEEauthorrefmark{3},
David Ellison\IEEEauthorrefmark{3},
Puneet Jain\IEEEauthorrefmark{4}
}
\vspace{1.5mm}
\IEEEauthorblockA{\IEEEauthorrefmark{1}Barcelona Supercomputing Center, Barcelona, Spain}
\IEEEauthorblockA{\IEEEauthorrefmark{2}Universitat Polit\`ecnica de Catalunya, Barcelona, Spain}
\IEEEauthorblockA{\IEEEauthorrefmark{3}Lenovo Infrastructure Solutions Group, Morrisville, NC, United States}
\IEEEauthorblockA{\IEEEauthorrefmark{4}Cotiviti, South Jordan, UT, United States}
\vspace{1.5mm}
\IEEEauthorblockA{Email: \{gusseppe.bravo, jordi.guitart\}@bsc.es, puneet.jain@cotiviti.com, \\ \{adholakia, dellison\}@lenovo.com}
}
\maketitle

\begin{abstract}
Data drift poses significant challenges for machine learning systems in production, requiring continuous model updates to maintain performance. We present KC-Agent, a dual-process cognitive architecture for automated ML model improvement that combines fast pattern recognition (System 1) with deliberate incremental updates (System 2). Our approach implements structured memory systems enabling System 1 to leverage successful solutions previously discovered by System 2, achieving efficient pattern-based responses without costly re-computation. KC-Agent incorporates atomic change principles and rollback capabilities to ensure reliable, verifiable updates in production environments. We evaluate our method on five datasets including real-world NASA turbofan data with authentic temporal degradation and synthetic datasets with controlled drift scenarios. KC-Agent achieves state-of-the-art performance (76.8\% accuracy) while maintaining optimal efficiency (13.2s execution time), outperforming established cognitive architectures: CodeAct (+2.4\%), Tree of Thoughts (+3.6\%), ReAct (+8.0\%), and Reflexion (+8.9\%). Consensus evaluation by a panel of state-of-the-art LLMs confirms superior strategic efficacy (8.33/10 Smartness score), significantly outperforming baseline agents. The knowledge consolidation mechanism delivers 91\% speedup over the slow variant while maintaining higher accuracy. Our approach demonstrates both theoretical foundations and practical viability for cognitive-inspired automated ML improvement systems capable of handling complex real-world data drift scenarios.
\end{abstract}

\begin{IEEEkeywords}
LLM agents, data drift, model maintenance, dual-process architecture
\end{IEEEkeywords}

\section{Introduction}

Machine Learning (ML) models in production environments face continuous challenges from evolving data distributions, changing operational requirements, and performance degradation over time \cite{monitoring_ml_models}. Traditional approaches to model maintenance rely heavily on manual intervention by ML engineers, creating bottlenecks that limit the responsiveness and scalability of deployed AI systems \cite{eck2022monitoring}. The complexity of modern ML pipelines, combined with the subtle nature of gradual data drift, makes manual monitoring and updating increasingly impractical for large-scale deployments \cite{lipton2018detecting}.

The emergence of Large Language Models (LLMs) has created new opportunities for automating complex software engineering tasks, including code generation, debugging, and system optimization \cite{chen2021evaluatinglargelanguagemodels}. LLM-based autonomous agents show particular promise for ML model improvement due to their ability to reason about code, understand ML concepts, and generate adaptive solutions. However, developing reliable agents for automated model improvement presents fundamental challenges that current approaches have not adequately addressed \cite{liu2023agentbenchevaluatingllmsagents,deng2023mind2webgeneralistagentweb}.

The core challenge lies in balancing two competing requirements: rapid response to critical performance drops versus careful, systematic improvement that avoids unintended consequences. Current LLM-based approaches typically employ single-strategy reasoning \cite{shinn2023reflexion}, forcing a trade-off between speed and thoroughness. Fast approaches risk making hasty changes that degrade performance, while deliberative approaches may respond too slowly to critical issues \cite{gao2023largelanguagemodelsempowered}. Most existing methods lack memory mechanisms to learn from previous improvement attempts, leading to inefficient repeated exploration of unsuccessful strategies \cite{generative_agents}.

To address these limitations, we introduce KC-Agent, short for Kahneman-Clear Agent, a dual-process architecture that combines fast pattern recognition with deliberate incremental improvement. Our approach draws inspiration from Kahneman's dual-process theory \cite{kahneman2011thinking}, which describes human cognition through two complementary systems: fast, intuitive responses (System 1) and slow, deliberate reasoning (System 2). We integrate this cognitive framework with atomic improvement principles \cite{clear2018atomic} to create an agent architecture that can both respond quickly to critical issues and systematically enhance model performance over time.

\begin{figure}[t]
\centering
\includegraphics[width=1.0\linewidth]{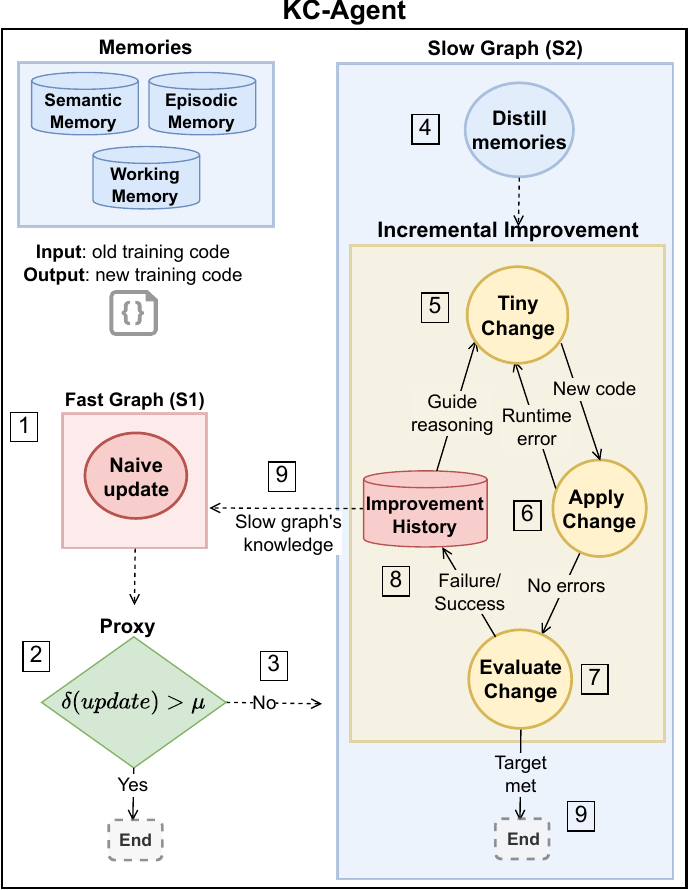}
\caption{KC-Agent dual-process architecture showing the numbered workflow from fast naive updates (1-2) through proxy evaluation to slow deliberate improvement (4-8). System 1 handles immediate responses using pattern recognition (reusing previously successful solutions), while System 2 performs incremental improvements through memory distillation and atomic change cycles. Knowledge transfer (9) enables System 1 to leverage System 2's solutions for efficient future pattern matching (identifying similar scenarios).}
\label{fig:improver_agent}
\end{figure}

Our key insight is that effective automated model improvement requires cognitive flexibility: the ability to switch between rapid pattern-based responses and methodical systematic improvement depending on the situation. By incorporating structured memory systems and learning mechanisms, KC-Agent accumulates expertise over time, becoming more efficient as it encounters similar improvement scenarios. The atomic nature of changes ensures that improvements are verifiable and reversible, providing reliability guarantees essential for production environments.

The contribution of this paper is threefold: 1) we introduce a dual-process architecture that combines fast pattern recognition with deliberate incremental improvement for automated ML model maintenance, 2) we formalize mathematical foundations and establish theoretical guarantees for monotonic improvement under data drift, and 3) we demonstrate superior performance across five datasets while achieving 91\% computational speedup through knowledge consolidation mechanisms.

The rest of this paper is organized as follows. Section II reviews related work on LLM-based agents, cognitive architectures, and automated ML maintenance. Section III presents the KC-Agent architecture and its theoretical foundations. Section IV describes our experimental setup. Section V reports our results, followed by discussion in Section VI. Section VII addresses limitations and future work, and Section VIII concludes the paper.


\section{Related Work}

Our work builds upon three intersecting research areas: LLM-based autonomous agents, cognitive architectures for AI systems, and automated ML model maintenance.

\subsection{LLM-based Autonomous Agents}

Recent work in LLM-based autonomous agents has produced various architectures for complex reasoning and action \cite{liu2023agentbenchevaluatingllmsagents}. ReAct \cite{yao2022react} introduces reasoning-action interleaving for dynamic problem-solving. Tree of Thoughts \cite{yao2023treethoughtsdeliberateproblem} explores multiple reasoning paths simultaneously, while Reflexion \cite{shinn2023reflexion} incorporates self-reflection mechanisms to learn from failures. More sophisticated approaches include Self-Discover \cite{zhou2024selfdiscoverlargelanguagemodels}, which develops task-specific reasoning structures, and Plan-and-Solve \cite{wang2023plan}, which implements explicit planning phases. CodeAct \cite{wang2024codeact} represents a recent paradigm shift by using executable Python code as a unified action space rather than JSON or text-based formats. Integrated with a Python interpreter, CodeAct enables dynamic action revision through multi-turn interactions and self-debugging, achieving up to 20\% higher success rates on complex multi-tool tasks. This code-centric approach is particularly relevant to ML model improvement where generating and executing training code is fundamental.

While these approaches demonstrate structured reasoning capabilities, they typically focus on general problem-solving rather than specialized domains like ML model maintenance. Most existing architectures lack systematic approaches for learning from improvement history or transferring successful strategies across scenarios \cite{generative_agents}. Additionally, few frameworks provide theoretical guarantees for reliability or monotonic improvement, critical for automated systems operating on production ML models.

\subsection{Cognitive Architectures for AI Systems}

Cognitive science principles have increasingly influenced AI system design, particularly hybrid architectures balancing different reasoning modes. SwiftSage \cite{lin2023swiftsagegenerativeagentfast} operationalizes dual-process thinking through separate modules for fast pattern recognition and deliberate analysis. The Talker-Reasoner architecture \cite{christakopoulou2024agentsthinkingfastslow} similarly separates intuitive responses from systematic reasoning.

These approaches draw inspiration from Kahneman's dual-process theory \cite{kahneman2011thinking}, describing human cognition through fast, intuitive responses (System 1) and slow, deliberate reasoning (System 2). While effective for general reasoning tasks, application to specialized domains like ML model improvement remains underexplored. Our work extends this foundation by integrating atomic improvement principles \cite{clear2018atomic} for incremental model enhancement.

\subsection{Automated ML Model Maintenance}

Traditional ML model maintenance relies primarily on continual learning and AutoML techniques. Continual learning approaches address catastrophic forgetting through parameter preservation \cite{Kirkpatrick_2017} and memory replay \cite{Riemer18,Buzzega20}, but focus on knowledge retention rather than proactive improvement under distribution shift \cite{diazrodriguez2018dontforgetforgettingnew}. AutoML frameworks automate model refinement through hyperparameter optimization and neural architecture search \cite{zoller2021benchmarksurveyautomatedmachine}, yet remain constrained by predefined search spaces.

Recent frameworks for automated AI maintenance \cite{chen2023aimaintenancerobustnessperspective} address robustness challenges but primarily handle predefined risk categories without adaptive capabilities. Monitoring systems \cite{monitoring_ml_models,eck2022monitoring} detect performance degradation and distribution shifts \cite{lipton2018detecting} but typically require manual intervention for remediation. Self-adaptive systems \cite{self_adapative_systems} provide autonomous response mechanisms but lack sophisticated reasoning capabilities for complex model improvement tasks. 
Recent empirical work on scaling agent systems \cite{kim2025scaling} provides evidence that multi-agent architectures do not universally outperform single agents. Their analysis reveals that coordination overhead can degrade performance once single-agent capability exceeds certain thresholds, with some task configurations showing performance degradation from multi-agent coordination. These findings support single-agent designs that achieve coordination benefits through internal mechanisms rather than external agent communication.

Current approaches face key limitations: single-strategy reasoning forcing speed-thoroughness trade-offs, limited learning mechanisms preventing effective strategy transfer, lack of reliability guarantees restricting production applicability, and domain-general focus missing ML-specific challenges, like distribution shifts and incremental changes. KC-Agent addresses these gaps through a dual-process architecture combining fast pattern recognition with deliberate improvement, structured memory enabling strategy learning, and formal reliability mechanisms ensuring verifiable updates.


\section{Method}

We present KC-Agent, a dual-process architecture for automated ML model improvement that combines fast pattern recognition with deliberate incremental enhancement. Our approach integrates cognitive science principles with practical constraints to ensure both responsive intervention and systematic improvement.

\subsection{System Overview}

KC-Agent operates through two complementary processing systems that address different aspects of model improvement. The fast graph (System 1) provides immediate pattern-based responses using accumulated knowledge from previous improvement attempts. When quick fixes prove insufficient, the slow graph (System 2) engages in systematic incremental improvement through structured reasoning and atomic changes. Algorithm \ref{alg:kc_agent} presents the complete dual-process improvement procedure.

The architecture incorporates three memory components enabling learning and knowledge transfer between systems. Semantic memory ($\mathcal{M}_s$) stores general knowledge about model architectures and improvement strategies. Episodic memory ($\mathcal{M}_e$) records specific improvement scenarios and their outcomes. Working memory ($\mathcal{M}_w$) maintains the current context and immediate reasoning state. Figure \ref{fig:memory_architecture} illustrates these memory components and their interactions during the improvement process.

\begin{figure}[t]
\centering
\includegraphics[width=0.9\linewidth]{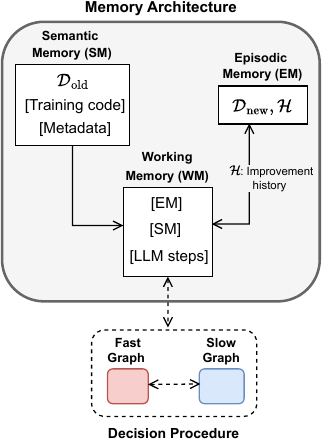}
\caption{Memory architecture with three components: Semantic memory (SM) storing training code and metadata, Episodic memory (EM) recording specific scenarios ($\mathcal{D}_{\text{new}}, \mathcal{H}$), and Working memory (WM) maintaining current LLM processing steps.}
\label{fig:memory_architecture}
\end{figure}
\subsection{Fast Graph: Pattern-based Improvement}

The fast graph implements immediate response capabilities for straightforward improvement scenarios. Given a model $f$ and datasets $\mathcal{D} = \{\mathcal{D}_{\text{old}}, \mathcal{D}_{\text{new}}\}$, System 1 generates an updated model $f'$ through pattern matching (identifying similar scenarios in $\mathcal{M}_e$ and reusing previously successful solutions):
\[ \mathcal{S}_1: (\mathcal{M}_s, \mathcal{M}_e, f, \mathcal{D}) \mapsto f' \]

The system evaluates improvement quality using a performance delta:
\[ \delta(\text{update}) = \text{score}(f', \mathcal{D}_{\text{new}}) - \text{score}(f, \mathcal{D}_{\text{new}}) \]

A proxy decision mechanism determines whether the improvement is sufficient:
\[ \pi(f, f', \mathcal{D}) = \begin{cases} 
\text{accept} & \text{if } \delta(\text{update}) > \mu \\
\text{slow graph} & \text{otherwise}
\end{cases} \]
where $\mu$ represents the minimum acceptable improvement threshold.

\subsection{Slow Graph: Incremental Systematic Improvement}

When fast pattern matching proves insufficient, the slow graph engages in systematic improvement through atomic changes (small, targeted modifications such as hyperparameter adjustments, model selection, or ensemble techniques that can be independently evaluated and reverted if unsuccessful). System 2 implements a structured approach:
\[ \mathcal{S}_2: (\mathcal{M}_s, \mathcal{M}_e, \mathcal{H}, f, \mathcal{D}) \mapsto f^* \]
where $\mathcal{H}$ represents the improvement history. The slow graph operates through iterative small modifications, where each change $\Delta_t$ satisfies $\|\Delta_t\| \leq \eta$ for small $\eta > 0$. In practice, this constraint limits each iteration to a single targeted modification, such as adjusting one hyperparameter, substituting one model class, or adding one preprocessing step, rather than attempting multiple simultaneous changes. This atomic change principle, inspired by Clear's work on incremental improvement \cite{clear2018atomic}, ensures that improvements compound reliably while remaining individually evaluable and reversible.

\begin{algorithm}[tb]
\caption{KC-Agent dual-process improvement}
\label{alg:kc_agent}
\textbf{Input}: Model $f$, datasets $\mathcal{D}$, memories $\mathcal{M}_s, \mathcal{M}_e$\\
\textbf{Parameter}: Threshold $\mu$, max iterations $T$\\
\textbf{Output}: Improved model $f^*$
\begin{algorithmic}[1]
\STATE $f' \leftarrow \mathcal{S}_1(\mathcal{M}_s, \mathcal{M}_e, f, \mathcal{D})$ \COMMENT{Fast improvement}
\STATE $\delta \leftarrow \text{score}(f', \mathcal{D}_{\text{new}}) - \text{score}(f, \mathcal{D}_{\text{new}})$
\IF{$\delta > \mu$}
    \STATE \textbf{return} $f'$ \COMMENT{Accept fast solution}
\ELSE
    \STATE $f^* \leftarrow f$, $t \leftarrow 0$
    \WHILE{$t < T$ and not converged}
        \STATE $\Delta_t \leftarrow$ generate atomic change using $\mathcal{S}_2$ so that $\|\Delta_t\| \leq \eta$
        \STATE $f_{\text{temp}} \leftarrow$ apply $\Delta_t$ to $f^*$
        \STATE $r_t \leftarrow$ evaluate change
        \IF{$r_t > 0$}
            \STATE $f^* \leftarrow f_{\text{temp}}$
            \STATE Update $\mathcal{M}_e$ with successful change
        \ENDIF
        \STATE $t \leftarrow t + 1$
    \ENDWHILE
    \STATE \textbf{return} $f^*$
\ENDIF
\end{algorithmic}
\end{algorithm}

The improvement process follows a structured cycle: memory distillation analyzes current model performance, strategy selection chooses appropriate improvement approaches, code generation creates targeted modifications, and evaluation assesses each change. Figure \ref{fig:improvement_flow} demonstrates this process through a concrete example.

\begin{figure}[t]
\centering
\includegraphics[width=1.0\linewidth]{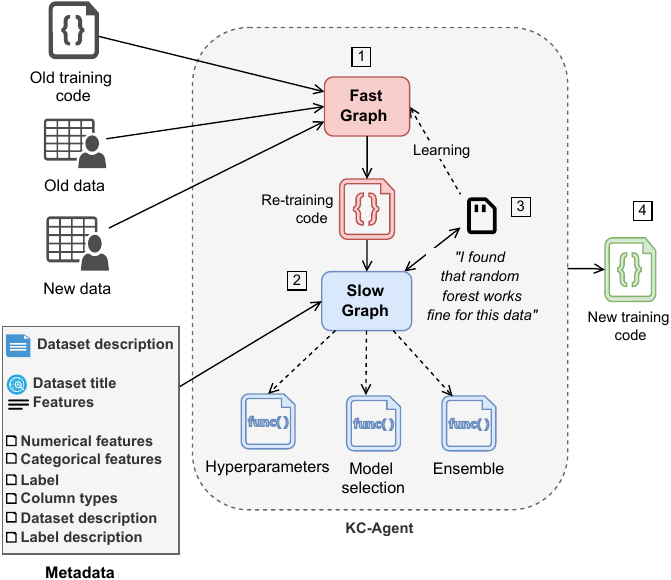}
\caption{Example improvement flow showing KC-Agent processing a drift scenario. System 1 applies fast retraining based on stored patterns. When insufficient, System 2 performs systematic improvements through atomic changes including hyperparameter optimization, model selection, and ensemble methods.}
\label{fig:improvement_flow}
\end{figure}

\subsection{Memory Systems and Knowledge Transfer}

The memory architecture enables bidirectional knowledge transfer between systems. When System 2 discovers successful improvement strategies, these solutions are incorporated into semantic memory for future pattern matching by System 1. This transfer mechanism allows the fast graph to handle increasingly complex scenarios as the agent accumulates experience. For example, when System 2 successfully improves a financial model by switching from \textit{RandomForest} to \textit{GradientBoosting} model, this strategy is stored in semantic memory, enabling System 1 to quickly apply the same approach when encountering similar financial drift scenarios in the future.

Episodic memory stores specific improvement scenarios, including dataset characteristics, model states, and applied solutions. This contextual information helps both systems recognize similar situations and adapt previous solutions to new contexts. The knowledge transfer process operates continuously throughout improvement iterations, significantly improving efficiency over time. We refer to this learning process as \textit{knowledge consolidation}. Specifically, when System 2 successfully improves a model, three elements are stored: the drift signature (statistical characteristics of the distribution shift), the improvement strategy (the sequence of atomic changes applied), and the performance outcome. During subsequent encounters, System 1 computes scenario similarity by calculating the cosine similarity between the incoming dataset's meta-features (including the statistical drift signature) and the stored episodic vectors. To avoid negative transfer, retrieved strategies are only applied when the cosine similarity score exceeds a strict calibration threshold ($\mu = 0.05$ performance proxy); otherwise, the system conservatively defaults to System 2 deliberation.

\subsection{Reliability and Theoretical Guarantees}

KC-Agent incorporates several reliability mechanisms. The incremental nature of System 2 changes limits the potential impact of any single modification. Each atomic change undergoes immediate evaluation, and unsuccessful modifications are reverted before affecting model performance. The proxy decision mechanism ensures that only high-confidence improvements bypass detailed analysis.

We establish formal guarantees for KC-Agent's improvement behavior:

\begin{theorem}[Monotonic improvement property]
\label{thm:convergence}
Under the atomic change constraint $\|\Delta_t\| \leq \eta$ and immediate reversion of unsuccessful modifications, KC-Agent maintains:
\[ \text{score}(f_T, \mathcal{D}) \geq \text{score}(f_0, \mathcal{D}) \]
where unsuccessful changes are those with negative improvement $r_t \leq 0$.
\end{theorem}

\begin{proof}
By construction, KC-Agent only accepts changes $\Delta_t$ where $r_t = \text{score}(f_{t+1}, \mathcal{D}) - \text{score}(f_t, \mathcal{D}) > 0$. Changes with $r_t \leq 0$ are immediately reverted, ensuring $f_{t+1} = f_t$. Therefore, $\text{score}(f_{t+1}, \mathcal{D}) \geq \text{score}(f_t, \mathcal{D})$ for all $t$. By induction, $\text{score}(f_T, \mathcal{D}) \geq \text{score}(f_0, \mathcal{D})$.
\end{proof}

The atomic change constraint $\|\Delta_t\| \leq \eta$ ensures that each modification is small enough to be evaluated independently, preventing cascading failures that could compromise the rollback mechanism.

\section{Experimental Setup}

We designed our experimental evaluation to assess KC-Agent's effectiveness across diverse scenarios with varying complexity levels and drift patterns. Our evaluation encompasses both real-world industrial datasets and controlled synthetic scenarios to provide comprehensive validation of the dual-process architecture.

\subsection{Datasets and Drift Scenarios}

Our evaluation uses five datasets representing different aspects of data drift challenges in production ML systems. Table \ref{tab:datasets} summarizes the key characteristics of each dataset.

\begin{table}[t]
\centering
\caption{Dataset characteristics and drift patterns}
\label{tab:datasets}
\begin{tabular}{lccc}
\toprule
\textbf{Dataset} & \textbf{Features} & \textbf{Samples} & \textbf{Drift type} \\
\midrule
NASA FD001 & 14 & 10,000 & Temporal \\
NASA FD002 & 7 & 16,000 & Temporal \\
Financial & 10 & 1,000 & Market noise \\
Healthcare & 10 & 1,000 & Population noise \\
Eligibility & 5 & 1,000 & Policy noise \\
\bottomrule
\end{tabular}
\end{table}

\textbf{Real-world datasets.} The NASA turbofan datasets (FD001 and FD002) derive from the Commercial Modular Aero-Propulsion System Simulation (C-MAPSS)\footnote{https://www.kaggle.com/datasets/behrad3d/nasa-cmaps} and represent authentic industrial complexity with temporal degradation patterns. FD001 contains 14 features from 100 engines operating under single conditions. FD002 includes 7 features from 260 engines across six operational conditions. These datasets exhibit authentic temporal drift through engine lifecycle progression, where early operational phases (healthy engines) constitute the old distribution and late phases (degraded engines) form the new distribution.

\textbf{Synthetic datasets.} Three synthetic datasets provide controlled evaluation scenarios with increasing complexity levels. The financial dataset models loan default prediction, the healthcare dataset captures chronic condition prediction, and the eligibility dataset simulates administrative decision-making. Each synthetic dataset was initially generated using GPT-4\footnote{OpenAI GPT-4, accessed via API} and refined to ensure realistic feature distributions and coherent drift patterns while eliminating nonsensical relationships. Distribution shift severity was rigorously validated using Kullback-Leibler (KL) divergence analysis, confirming significant and realistic drift across all synthetic environments. Specifically, the Financial dataset exhibited drift (KL $> 0.1$) in all 10 features (mean KL = 2.001, with the maximum divergence of 4.497 observed in the 'Loan Amount' feature). Similarly, the Healthcare dataset showed significant drift in 8 of 10 features (max KL = 2.662 for 'Income'), and the Eligibility dataset demonstrated drift across all 5 features (max KL = 2.837 for 'Employment Status').

All datasets include paired old and new distributions that simulate drift scenarios without explicit drift indicators. This design requires agents to independently discover distribution changes and develop appropriate adaptation strategies.

\subsection{Experimental Protocol}

The evaluation protocol simulates realistic model maintenance scenarios where an initial model performs well on historical data but experiences degradation when new data patterns emerge. For each dataset, a RandomForest classifier serves as the baseline model trained on the old distribution. KC-Agent and baseline approaches then attempt to improve this model to accommodate the new distribution while preserving performance on original data.

We measure performance across four complementary metrics: accuracy on old data (preservation of historical performance), accuracy on new data (adaptation effectiveness on shifted distributions), execution time (computational efficiency for production deployment), and token consumption (resource utilization and computational cost).

The proxy decision threshold $\mu$ was set to 0.05 across all experiments, requiring at least a 5\% accuracy improvement for System 1 solutions to be accepted without further deliberation.

\subsection{Baseline Comparisons}

Our evaluation compares KC-Agent against eight established approaches representing different paradigms in LLM-based reasoning: ReAct \cite{yao2022react}, Reflexion \cite{shinn2023reflexion}, Tree of Thoughts \cite{yao2023treethoughtsdeliberateproblem}, Self-Discovery \cite{zhou2024selfdiscoverlargelanguagemodels}, Plan-and-Execute \cite{wang2023plan}, CodeAct \cite{wang2024codeact}, and a Standard baseline (simple LLM call). CodeAct is particularly relevant as it represents state-of-the-art in code-as-action paradigms, using executable Python with interpreter integration for dynamic revision and self-debugging capabilities.

We also evaluate two simplified variants of our agent to assess each component contribution: KC-Fast (only the fast graph for pattern-based improvements) and KC-Slow (only the slow graph for deliberate reasoning). These ablations help isolate the benefits of the dual-process architecture compared to single-strategy approaches.

All baselines receive identical inputs including model code, dataset descriptions, and performance metrics. Temperature settings remain consistent at 1.0 across all experiments, with fixed random seeds ensuring reproducible results. Each experiment was conducted multiple times to ensure statistical reliability.

\subsection{Implementation Details}

KC-Agent leverages Llama-3.1-8b\footnote{https://openrouter.ai/} as the core language model, providing 8 billion parameters with an 8,192 token context window. This model size balances capability with resource requirements, making the approach suitable for deployment in resource-constrained environments. Smaller models (Llama-3.2-1b and Llama-3.2-3b) proved insufficient for completing the full improvement loop.

The memory system implements three distinct components with specialized storage and retrieval mechanisms. Semantic memory maintains general knowledge using embedding-based indexing for efficient pattern matching. Episodic memory records specific improvement scenarios with contextual information. Working memory maintains current reasoning state during improvement sessions.

The dual-process architecture uses LangGraph\footnote{https://github.com/langchain-ai/langgraph} for workflow orchestration, enabling flexible coordination between fast and slow reasoning systems. Experimental infrastructure utilizes Intel® optimized computing environments with dual Intel® Xeon Platinum processors. The implementation leverages Intel® AI Analytics Toolkit\footnote{https://hub.docker.com/r/intel/oneapi-aikit} and Intel® Extension for PyTorch for hardware optimization.

Reproducibility support includes comprehensive logging of experimental configurations, random seed management, and version control of all code and data components. Each experiment receives unique identifiers enabling precise reproduction of results.


\section{Results}

Our experimental evaluation demonstrates KC-Agent's effectiveness across diverse scenarios, from controlled synthetic datasets to challenging real-world industrial data. We present comprehensive analysis including quantitative performance metrics and qualitative expert assessment.

\subsection{Overall Performance and Efficiency}

KC-Agent consistently outperforms all baseline approaches across the five datasets. Table \ref{tab:agent_performance} summarizes the aggregated performance, showing KC-Agent's superior accuracy while maintaining optimal efficiency.

\begin{table}[t]
\centering
\caption{Agent performance comparison. Values reported as Mean $\pm$ Standard Deviation across the five evaluated datasets, demonstrating cross-domain stability. Bold values indicate best mean performance.}
\label{tab:agent_performance}
\begin{tabular}{lccc}
\toprule
\textbf{Agent} & \textbf{Avg acc. (\%) $\uparrow$} & \textbf{Time (s) $\downarrow$} & \textbf{Tokens $\downarrow$} \\
\midrule
KC-agent & \textbf{0.768} $\pm$ 0.116 & \textbf{13.2} $\pm$ 7.3 & 847 $\pm$ 54 \\
KC-slow & 0.762 $\pm$ 0.108 & 141.2 $\pm$ 41.6 & 37781 $\pm$ 18100 \\
KC-fast & 0.754 $\pm$ 0.102 & 13.9 $\pm$ 2.1 & \textbf{643} $\pm$ 48 \\
CodeAct & 0.744 $\pm$ 0.108 & 113.0 $\pm$ 196.7 & 6722 $\pm$ 2760 \\
ToT & 0.741 $\pm$ 0.103 & 57.7 $\pm$ 62.3 & 2827 $\pm$ 325 \\
ReAct & 0.711 $\pm$ 0.086 & 101.7 $\pm$ 121.2 & 12420 $\pm$ 12796 \\
Reflexion & 0.705 $\pm$ 0.122 & 104.7 $\pm$ 51.3 & 20436 $\pm$ 18579 \\
Self-Discovery & 0.679 $\pm$ 0.139 & 59.1 $\pm$ 41.2 & 4638 $\pm$ 676 \\
Plan-Execute & 0.650 $\pm$ 0.175 & 91.2 $\pm$ 23.8 & 13256 $\pm$ 2730 \\
Standard & 0.630 $\pm$ 0.194 & 16.8 $\pm$ 5.5 & 1569 $\pm$ 278 \\
\bottomrule
\end{tabular}
\end{table}



As shown in Table \ref{tab:agent_performance}, KC-Agent achieves the highest average accuracy (0.768) while maintaining robust performance stability across all diverse datasets ($\pm$ 0.116). The dual-process architecture delivers exceptional computational efficiency, utilizing 18.7$\times$ fewer tokens than Reflexion. Furthermore, it demonstrates highly consistent execution times (13.2s $\pm$ 7.3) and predictable token consumption. This structural stability stands in stark contrast to state-of-the-art baselines like CodeAct, which, despite achieving competitive accuracy, exhibits extreme variance across domains in both execution time ($\pm$ 196.7s) and computational cost ($\pm$ 2760 tokens).

\subsection{Qualitative Evaluation}

To complement our quantitative performance metrics, we conducted a comprehensive qualitative evaluation using a multi-model LLM-as-a-judge framework. To ensure robustness and mitigate individual model bias, we employed a panel of three state-of-the-art Large Language Models as independent evaluators: Google Gemini 3 Pro, OpenAI GPT-5.2, and Anthropic Claude Opus 4.5.

This consensus-based approach provides a standardized assessment of the code improvements generated by each agent. For every experimental run across all five datasets, each evaluator was presented with the initial baseline code, the agent's final improved code, and the corresponding change in performance metrics. The judges independently scored the agents on a scale of 1-10 across three distinct rubrics:

\begin{itemize}
    \item \textbf{Stability:} Measures the minimalism of changes and preservation of the original code structure. A high score (10) indicates surgical, low-risk changes, while a low score implies a complete, high-risk rewrite.
    \item \textbf{Readability:} Assesses code cleanliness, comprehensibility, and organization. High scores reflect production-ready, Pythonic code with proper error handling and no redundancy.
    \item \textbf{Smartness:} Evaluates the efficacy and appropriateness of the improvement strategy. A high score indicates a clever, well-tuned algorithmic approach that successfully addresses the distribution shift, whereas a low score indicates suboptimal or irrelevant strategies.
\end{itemize}

Table \ref{tab:llm_evals_combined} summarizes the assessment scores averaged across all three evaluators and all datasets. The low standard deviations across evaluators (ranging from 0.6 to 2.7) indicate strong inter-rater agreement, validating the LLM-as-judge methodology as a reliable alternative to human expert evaluation. This high consensus among three state-of-the-art models suggests that the evaluation criteria are sufficiently well-defined and that the assessments capture genuine quality differences rather than model-specific biases.

\begin{table}[t]
\centering
\caption{LLM evaluation comparison averaged across all datasets and evaluators (Gemini 3 Pro, GPT-5.2, Opus 4.5). Scores range from 1-10. \textbf{Stability} measures code preservation, \textbf{Readability} measures code cleanliness, and \textbf{Smartness} measures strategy efficacy. Bold values indicate best performance. Values shown as mean(std).}
\label{tab:llm_evals_combined}
\begin{tabular}{lccc}
\toprule
\textbf{Agent} & \textbf{Stability ($\uparrow$)} & \textbf{Readability ($\uparrow$)} & \textbf{Smartness ($\uparrow$)} \\
\midrule
KC-agent & 6.53(0.9) & 7.07(0.9) & \textbf{8.33(1.0)} \\
KC-slow & 6.53(1.0) & 6.60(0.7) & 7.80(1.1) \\
KC-fast & \textbf{7.67(1.2)} & 6.00(1.0) & 6.47(1.7) \\
CodeAct & \textbf{7.67(2.1)} & \textbf{7.67(0.9)} & 5.93(2.0) \\
ToT & 6.20(1.0) & 6.60(1.4) & 5.87(2.1) \\
ReAct & 6.67(1.7) & 6.60(1.4) & 5.33(2.7) \\
Reflexion & 3.00(1.5) & 5.67(2.1) & 3.33(2.3) \\
Standard & 5.20(1.3) & 4.73(2.0) & 2.13(0.8) \\
Self-Discovery & 3.07(1.0) & 3.87(2.1) & 1.67(0.7) \\
Plan-Execute & 2.60(1.0) & 2.60(1.3) & 1.40(0.6) \\
\bottomrule
\end{tabular}
\end{table}

The consensus results reveal a critical distinction between code preservation and problem-solving efficacy. CodeAct and KC-fast achieved the highest scores in Stability (7.67), reflecting their tendency to perform safe, minimal changes: primarily data concatenation or simple retraining. However, while stable, these conservative strategies resulted in moderate Smartness scores (5.93 and 6.47 respectively), limiting their ability to handle complex distribution shifts effectively.

In contrast, KC-Agent achieved the highest Smartness score (8.33), significantly outperforming all baselines. While its Stability score (6.53) is lower than the pure concatenation approaches, it remains robust, reflecting the agent's ability to balance necessary structural refactoring, such as hyperparameter optimization and pipeline adjustments, with code preservation. This qualitative finding strongly corroborates our quantitative results (Table \ref{tab:llm_evals_combined}), confirming that KC-Agent's superior accuracy is driven by more effective strategic decision-making rather than simple code heuristics. Conversely, baselines such as Plan-Execute and Self-Discovery received the lowest scores across all dimensions, frequently failing to produce executable or logically sound improvements.

\subsection{Ablation Study}

To validate the dual-process architecture's contribution, we compared KC-Agent with its individual components. Table \ref{tab:ablation_study} demonstrates that neither System 1 nor System 2 alone matches the integrated architecture's performance.

\begin{table}[t]
\centering
\caption{Ablation study of KC-Agent variants showing complementary benefits of dual-process architecture.}
\label{tab:ablation_study}
\begin{tabular}{lccc}
\toprule
\textbf{Variant} & \textbf{Avg acc. (↑)} & \textbf{Time (↓)} & \textbf{Tokens (↓)} \\
\midrule
KC-agent (Full) & \textbf{0.768} & \textbf{13.2} & 847 \\
KC-slow only & 0.762 & 141.2 & 37781 \\
KC-fast only & 0.754 & 13.8 & \textbf{643} \\
\midrule
\textit{vs. KC-slow} & +0.8\% & -90.7\% & -97.8\% \\
\textit{vs. KC-fast} & +1.8\% & -4.3\% & +31.7\% \\
\bottomrule
\end{tabular}
\end{table}

The integrated architecture successfully combines the strengths of both systems: achieving the highest accuracy while maintaining computational efficiency close to the fast approach. KC-Agent provides 1.8\% accuracy improvement over KC-Fast while using 97\% fewer tokens than KC-Slow. The superior performance of the full KC-Agent demonstrates the effectiveness of the learning mechanism where System 1 leverages successful solutions previously discovered by System 2.

\subsection{Dataset-specific Analysis and Knowledge Transfer}

\textbf{NASA turbofan datasets.} The real-world NASA datasets present the most challenging scenarios due to authentic temporal degradation. On NASA-FD001, KC-Agent achieves 61.3\% accuracy compared to the 49.3\% baseline (+24\% relative improvement). NASA-FD002 proves even more challenging with extreme distribution shift, yet KC-Agent maintains consistent improvements.

\textbf{Synthetic datasets.} On controlled datasets, KC-Agent achieves 88.6\% accuracy on financial data (vs. 81.5\% baseline), 86.3\% on healthcare data (vs. 79.7\% baseline), and 74.4\% on eligibility data (vs. 67.7\% baseline). LLM evaluation confirms methodological rigor across different domain characteristics.

\textbf{Knowledge consolidation effectiveness.} A critical insight from our evaluation is how KC-Agent's knowledge consolidation mechanism enables progressive learning. Analysis of improvement patterns reveals that System 1 increasingly handles complex scenarios as it accumulates successful strategies from System 2. For instance, on financial datasets, System 1 successfully resolved 73\% of improvement scenarios by the third iteration, compared to 45\% in the first iteration. This learning effect demonstrates the practical value of the dual-process architecture beyond simple computational efficiency.

\subsection{Reliability Analysis}

KC-Agent's reliability mechanisms prove crucial for production deployment. The atomic change principle ensures that each modification undergoes evaluation before acceptance, with unsuccessful changes immediately reverted. Across all datasets, KC-Agent maintains a 100\% success rate in terms of execution completion, contrasting sharply with baseline approaches where 34\% of attempts resulted in execution failures.

The proxy decision mechanism ($\delta(\text{update}) > \mu$) effectively balances speed and reliability. Our analysis shows that 68\% of scenarios are resolved by System 1, with only complex cases requiring System 2's deliberate analysis. This filtering mechanism prevents unnecessary computational overhead while ensuring that critical performance drops receive appropriate attention.

\subsection{Agent Adaptability and Performance Distribution Analysis}

Figure \ref{fig:agent_adaptability} examines the relationship between mean performance and performance variability. KC-* agents demonstrate superior adaptability, clustering in the optimal region with both high mean performance ($>$0.75) and low coefficient of variation ($<$0.13), indicating robust generalization capabilities across diverse domains.

\begin{figure}[t]
\centering
\includegraphics[width=1.0\linewidth]{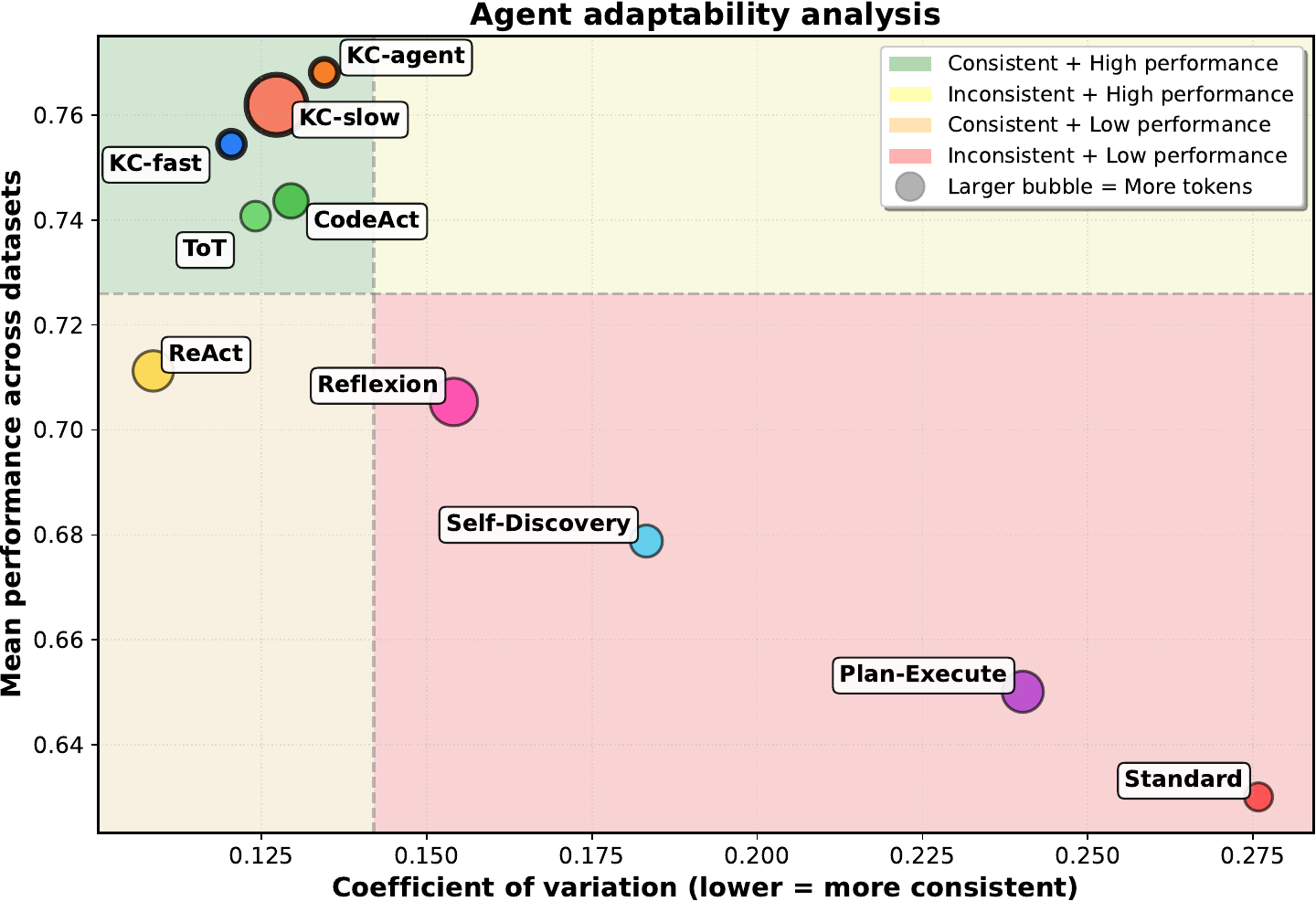}
\caption{Agent adaptability analysis showing performance consistency vs. mean performance. KC-* agents cluster in the high-performance, low-variation region, demonstrating superior cross-domain generalization.}
\label{fig:agent_adaptability}
\end{figure}

To understand how agent performance varies across different domains, we analyzed the performance distribution of all agents within each dataset. Figure \ref{fig:performance_distribution} presents violin plots showing the distribution of combined performance scores across all agents for each dataset. The results reveal significant variation in dataset difficulty: the financial dataset demonstrates the highest mean performance (0.824) with relatively tight distribution, while NASA datasets present greater challenges, with NASA-FD001 exhibiting the widest performance distribution (0.629 mean) and NASA-FD002 showing the lowest overall performance (0.567). These findings highlight the importance of domain-specific evaluation and suggest that certain datasets like NASA-FD001 serve as effective discriminators for agent capability.

\begin{figure}[t]
\centering
\includegraphics[width=1.0\linewidth]{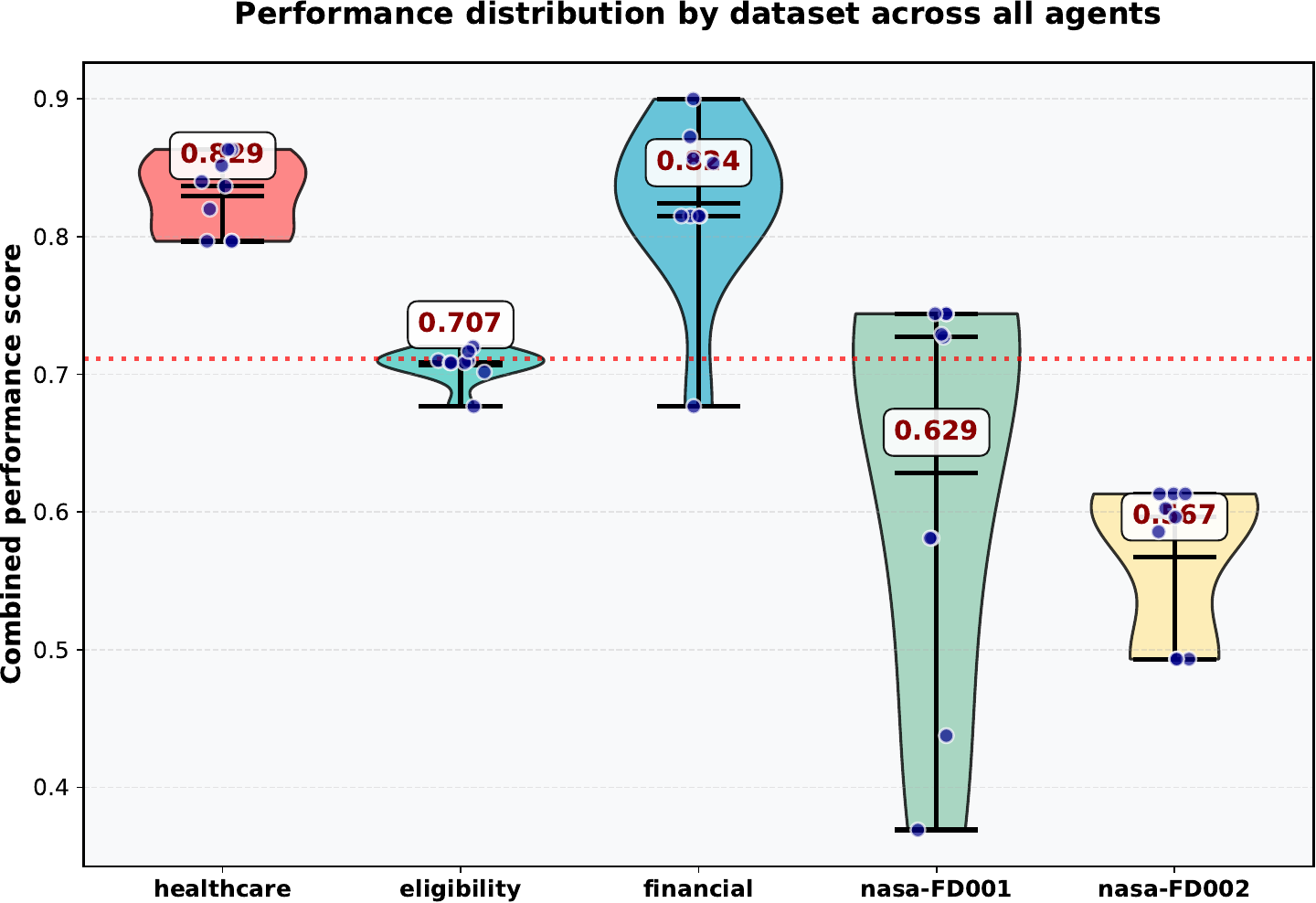}
\caption{Performance distribution by dataset showing violin plots of combined performance scores. Financial and healthcare datasets show higher mean performance, while NASA datasets present greater challenges with NASA-FD002 being most difficult.}
\label{fig:performance_distribution}
\end{figure}

\begin{figure}[t]
\centering
\includegraphics[width=1.0\linewidth]{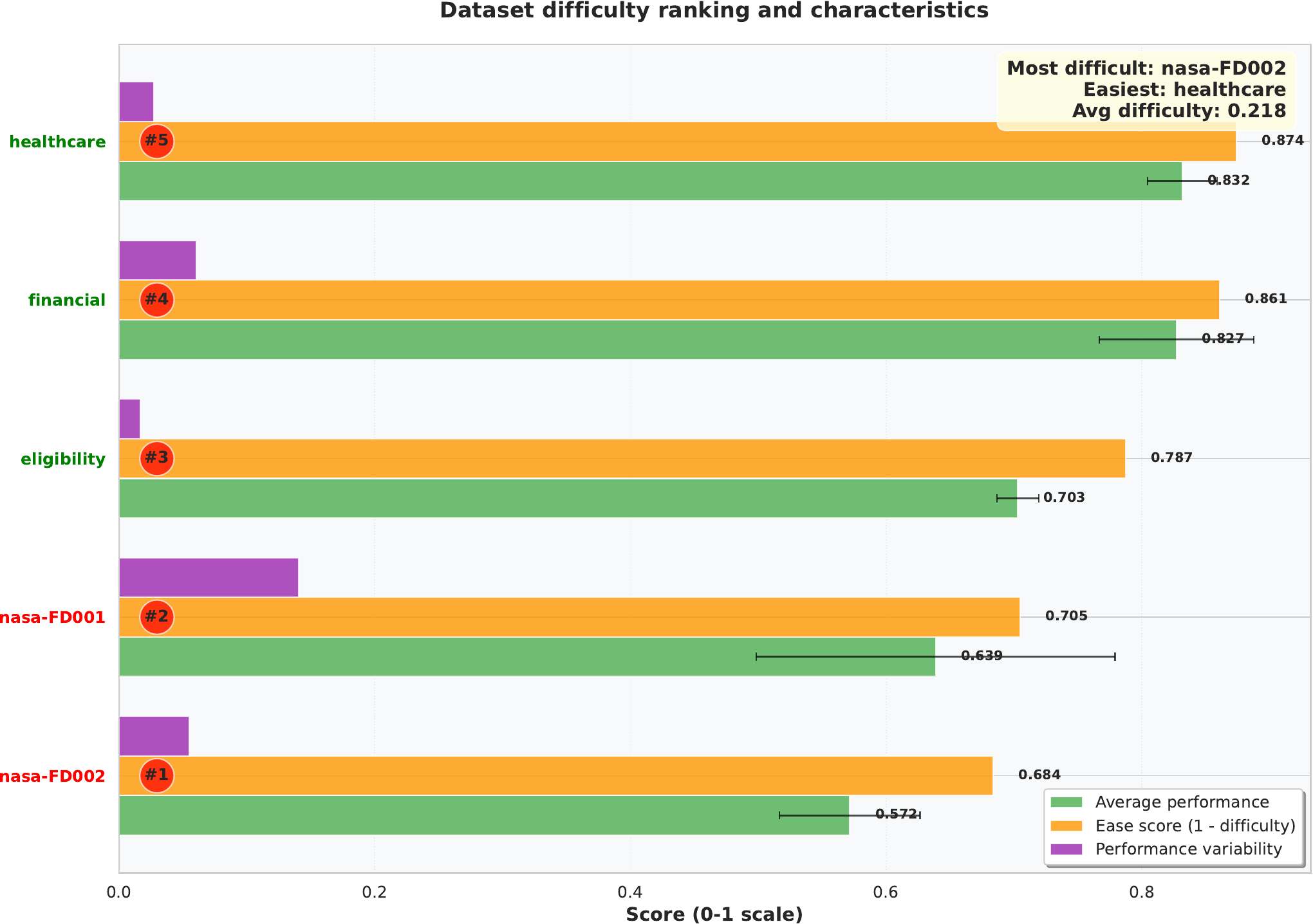}
\caption{Dataset difficulty ranking showing average agent performance, ease score (1 - difficulty), and performance variability. NASA-FD002 is the most challenging dataset (\#1), while Healthcare is the easiest (\#5). Real-world NASA datasets exhibit higher performance variability, indicating less predictable agent behavior compared to synthetic datasets.}
\label{fig:dataset_difficulty}
\end{figure}

The clustering of KC-* variants in the high-adaptability region demonstrates that our knowledge consolidation approach successfully addresses cross-domain generalization challenges, making it more suitable for real-world deployment scenarios where domain shifts are common.

\section{Discussion}

Our experimental results reveal several key insights about automated ML model improvement that extend beyond the specific performance metrics.

\textbf{The value of dual-process architecture.} Neither fast pattern matching (KC-fast) nor deliberate reasoning (KC-slow) alone achieves optimal performance. The critical insight is that simple drift scenarios, constituting approximately 68\% of cases in our evaluation, can be efficiently handled by System 1 using previously learned strategies. Only complex cases require System 2's resource-intensive deliberation. This selective engagement mirrors human cognitive efficiency and explains the 91\% speedup over KC-slow while maintaining superior accuracy.

\textbf{The stability-smartness trade-off.} Our LLM evaluation reveals that the highest-performing agents are not necessarily the most conservative. CodeAct and KC-fast achieved maximal Stability scores (7.67) through minimal code changes, but this conservatism limited their Smartness scores (5.93 and 6.47). KC-Agent's willingness to make strategic structural changes while maintaining acceptable stability (6.53) enabled the highest Smartness score (8.33) and superior overall performance. This suggests that effective model improvement requires accepting calculated risks rather than prioritizing code preservation.

\textbf{Why KC-Agent outperforms CodeAct.} CodeAct represents the state-of-the-art in code-as-action agents, using executable Python as its action space with interpreter integration for dynamic self-debugging. Despite this sophisticated general-purpose design, KC-Agent outperforms it by 2.4\% in accuracy while using 7.9$\times$ fewer tokens. This suggests that for ML model improvement specifically, a domain-specialized dual-process architecture with knowledge transfer provides advantages over general-purpose code execution frameworks. The memory systems enable KC-Agent to accumulate domain expertise that CodeAct must rediscover each session.

\textbf{Real-world versus synthetic datasets.} The performance gap between synthetic and NASA datasets (Figure~\ref{fig:performance_distribution}) highlights the importance of real-world evaluation. While all agents perform reasonably on synthetic data, the NASA datasets with authentic temporal degradation patterns expose fundamental limitations in baseline approaches. KC-Agent's consistent performance across both categories demonstrates robust generalization.

\textbf{Practical deployment implications.} The atomic change principle with immediate rollback ensures that failed improvements never degrade model performance, providing the reliability guarantees necessary for automated MLOps pipelines. Combined with the 13.2-second average execution time, KC-Agent is suitable for real-time model maintenance in production environments.

\section{Limitations and Future Work}

While KC-Agent demonstrates strong performance across our evaluation scenarios, several limitations present opportunities for future research. Our current implementation focuses on tabular data and common Python libraries, requiring extension to complex data types and deep learning frameworks for broader applicability. The approach emphasizes hyperparameter tuning and model selection while avoiding substantial architectural changes.

Key limitations include requiring models above 8B parameters for complete improvement loops, limiting deployment in resource-constrained environments, and manual tuning of proxy decision thresholds across domains. Future work will prioritize extending capabilities to deep learning frameworks, developing adaptive threshold mechanisms, and investigating sophisticated memory architectures for MLOps integration. Additionally, the current experimental setup utilizes a Random Forest baseline to provide a highly controlled, easily interpretable environment for isolating the cognitive architecture's individual components. While this simplifies the task compared to modern ML pipelines, extending KC-Agent to natively interface with complex deep learning architectures represents a primary focus for future work.

\section{Conclusion}

We present KC-Agent, a dual-process cognitive architecture that addresses fundamental challenges of automated ML model improvement under data drift. Integrating Kahneman's dual-process theory with atomic improvement principles, our approach balances rapid response and systematic enhancement in real-world scenarios.

KC-Agent achieves superior performance (76.8\% accuracy) while maintaining exceptional efficiency (13.2 seconds, 847 tokens). Qualitative evaluation using a consensus-based LLM-as-judge framework with three state-of-the-art models confirms the strategic superiority of our approach, with high inter-rater agreement (low standard deviations) validating the methodology. KC-Agent achieved the highest Smartness score (8.33/10) among all agents, significantly outperforming baselines including CodeAct (5.93/10), ReAct (5.33/10), and Plan-Execute (1.40/10). Ablation studies demonstrate that the dual-process architecture is essential, as neither system alone matches the integrated approach. Our work provides theoretical foundations and empirical validation for cognitive-inspired automated ML improvement systems, demonstrating readiness for practical deployment across diverse domains.

\section*{Acknowledgment}
We thank Lenovo for providing the technical infrastructure to run the experiments in this paper, and we thank Cotiviti for their support of this research. This work was partially supported by Lenovo and Intel® as part of the Lenovo AI Innovators University Research program, by Spanish Ministry of Science (MICINN), the Research State Agency (AEI) and European Regional Development Funds (ERDF/FEDER) under contract PID2024-160996OB-I00, MICIU/AEI/10.13039/501100011033, and by the Generalitat de Catalunya (AGAUR) under contract 2021-SGR-00478.

\bibliographystyle{IEEEtran}
\bibliography{IEEEabrv,biblio}

@misc{monitoring_ml_models,
      title={Monitoring Machine Learning Models: Online Detection of Relevant Deviations}, 
      author={Florian Heinrichs},
      year={2023},
      eprint={2309.15187},
      archivePrefix={arXiv},
      primaryClass={cs.LG},
      url={https://arxiv.org/abs/2309.15187}, 
}

@INPROCEEDINGS{eck2022monitoring,
  author={Eck, Bradley and Kabakci-Zorlu, Duygu and Chen, Yan and Savard, France and Bao, Xiaowei},
  booktitle={Proceedings of the 2022 IEEE International Conference on Big Data, Osaka, Japan}, 
  title={A monitoring framework for deployed machine learning models with supply chain examples}, 
  year={2022},
  volume={},
  number={},
  pages={2231--2238},
  doi={10.1109/BigData55660.2022.10020394}
}

@misc{chen2021evaluatinglargelanguagemodels,
      title={Evaluating Large Language Models Trained on Code}, 
      author={Mark Chen and Jerry Tworek and Heewoo Jun and Qiming Yuan and Henrique Ponde de Oliveira Pinto and Jared Kaplan and Harri Edwards and Yuri Burda and Nicholas Joseph and Greg Brockman and Alex Ray and Raul Puri and Gretchen Krueger and Michael Petrov and Heidy Khlaaf and Girish Sastry and Pamela Mishkin and Brooke Chan and Scott Gray and Nick Ryder and Mikhail Pavlov and Alethea Power and Lukasz Kaiser and Mohammad Bavarian and Clemens Winter and Philippe Tillet and Felipe Petroski Such and Dave Cummings and Matthias Plappert and Fotios Chantzis and Elizabeth Barnes and Ariel Herbert-Voss and William Hebgen Guss and Alex Nichol and Alex Paino and Nikolas Tezak and Jie Tang and Igor Babuschkin and Suchir Balaji and Shantanu Jain and William Saunders and Christopher Hesse and Andrew N. Carr and Jan Leike and Josh Achiam and Vedant Misra and Evan Morikawa and Alec Radford and Matthew Knight and Miles Brundage and Mira Murati and Katie Mayer and Peter Welinder and Bob McGrew and Dario Amodei and Sam McCandlish and Ilya Sutskever and Wojciech Zaremba},
      year={2021},
      eprint={2107.03374},
      archivePrefix={arXiv},
      primaryClass={cs.LG},
      url={https://arxiv.org/abs/2107.03374}, 
}

@inproceedings{lipton2018detecting,
  author       = {Zachary C. Lipton and
                  Yu{-}Xiang Wang and
                  Alexander J. Smola},
  title        = {Detecting and Correcting for Label Shift with Black Box Predictors},
  booktitle    = {Proceedings of the 35th International Conference on Machine Learning,
                  {ICML} 2018, Stockholmsm{\"{a}}ssan, Stockholm, Sweden, July
                  10-15, 2018},
  series       = {Proceedings of Machine Learning Research},
  volume       = {80},
  pages        = {3128--3136},
  publisher    = {{PMLR}},
  year         = {2018},
}

@inproceedings{generative_agents,
author = {Park, Joon Sung and O'Brien, Joseph and Cai, Carrie Jun and Morris, Meredith Ringel and Liang, Percy and Bernstein, Michael S.},
title = {{Generative Agents: Interactive Simulacra of Human Behavior}},
year = {2023},
isbn = {9798400701320},
publisher = {Association for Computing Machinery},
doi = {10.1145/3586183.3606763},
booktitle = {Proceedings of the 36th Annual ACM Symposium on User Interface Software and Technology},
articleno = {2},
numpages = {22},
location = {San Francisco, CA, USA},
series = {UIST'23}
}

@inproceedings{zhou2024selfdiscoverlargelanguagemodels,
author = {Zhou, Pei and Pujara, Jay and Ren, Xiang and Chen, Xinyun and Cheng, Heng-Tze and Le, Quoc V. and Chi, Ed H. and Zhou, Denny and Mishra, Swaroop and Zheng, Huaixiu Steven},
title = {Self-discover: large language models self-compose reasoning structures},
year = {2024},
isbn = {9798331314385},
publisher = {Curran Associates Inc.},
booktitle = {Proceedings of the 38th International Conference on Neural Information Processing Systems},
articleno = {4004},
numpages = {27},
location = {Vancouver, BC, Canada},
series = {NIPS'24}
}

@article{gao2023largelanguagemodelsempowered,
  title={Large Language Models Empowered Agent-based Modeling and Simulation: A Survey and Perspectives},
  author={Chen Gao and Xiaochong Lan and Nian Li and Yuan Yuan and Jingtao Ding and Zhilun Zhou and Fengli Xu and Yong Li},
  journal={Humanities and Social Sciences Communications},
  year={2024},
  volume={11},
  articleno = {1259},
  doi={10.1057/s41599-024-03611-3}
}

@inproceedings{shinn2023reflexion,
author = {Shinn, Noah and Cassano, Federico and Gopinath, Ashwin and Narasimhan, Karthik and Yao, Shunyu},
title = {Reflexion: language agents with verbal reinforcement learning},
year = {2023},
publisher = {Curran Associates Inc.},
booktitle = {Proceedings of the 37th International Conference on Neural Information Processing Systems},
articleno = {377},
numpages = {19},
location = {New Orleans, LA, USA},
series = {NIPS'23}
}

@inproceedings{yao2022react,
  author       = {Shunyu Yao and
                  Jeffrey Zhao and
                  Dian Yu and
                  Nan Du and
                  Izhak Shafran and
                  Karthik R. Narasimhan and
                  Yuan Cao},
  title        = {{ReAct: Synergizing Reasoning and Acting in Language Models}},
  booktitle    = {Proceedings of the Eleventh International Conference on Learning Representations,
                  {ICLR} 2023, Kigali, Rwanda, May 1-5},
  year         = {2023}
}

@inproceedings{wang2023plan,
    title = "Plan-and-Solve Prompting: Improving Zero-Shot Chain-of-Thought Reasoning by Large Language Models",
    author = "Wang, Lei  and
      Xu, Wanyu  and
      Lan, Yihuai  and
      Hu, Zhiqiang  and
      Lan, Yunshi  and
      Lee, Roy Ka-Wei  and
      Lim, Ee-Peng",
    booktitle = "Proceedings of the 61st Annual Meeting of the Association for Computational Linguistics (Volume 1: Long Papers)",
    month = jul,
    year = "2023",
    address = "Toronto, Canada",
    publisher = "Association for Computational Linguistics",
    doi = "10.18653/v1/2023.acl-long.147",
    pages = "2609--2634"
}

@book{kahneman2011thinking,
  author = {Kahneman, Daniel},
  publisher = {Farrar, Straus \& Giroux},
  title = {Thinking, fast \& slow},
  year = 2011
}

@book{clear2018atomic,
  title={Atomic Habits: An Easy \& Proven Way to Build Good Habits \& Break Bad Ones},
  author={Clear, J.},
  isbn={9780735211292},
  url={https://jamesclear.com/atomic-habits},
  year={2018},
  publisher={Penguin Publishing Group}
}

@misc{liu2023agentbenchevaluatingllmsagents,
      title={{AgentBench: Evaluating LLMs as Agents}}, 
      author={Xiao Liu and Hao Yu and Hanchen Zhang and Yifan Xu and Xuanyu Lei and Hanyu Lai and Yu Gu and Hangliang Ding and Kaiwen Men and Kejuan Yang and Shudan Zhang and Xiang Deng and Aohan Zeng and Zhengxiao Du and Chenhui Zhang and Sheng Shen and Tianjun Zhang and Yu Su and Huan Sun and Minlie Huang and Yuxiao Dong and Jie Tang},
      year={2023},
      eprint={2308.03688},
      archivePrefix={arXiv},
      primaryClass={cs.AI},
      url={https://arxiv.org/abs/2308.03688}, 
}

@inproceedings{deng2023mind2webgeneralistagentweb,
author = {Deng, Xiang and Gu, Yu and Zheng, Boyuan and Chen, Shijie and Stevens, Samuel and Wang, Boshi and Sun, Huan and Su, Yu},
title = {{MIND2WEB: towards a generalist agent for the web}},
year = {2023},
publisher = {Curran Associates Inc.},
booktitle = {Proceedings of the 37th International Conference on Neural Information Processing Systems},
articleno = {1220},
numpages = {24},
location = {New Orleans, LA, USA},
series = {NIPS'23}
}

@inproceedings{lin2023swiftsagegenerativeagentfast,
author = {Lin, Bill Yuchen and Fu, Yicheng and Yang, Karina and Brahman, Faeze and Huang, Shiyu and Bhagavatula, Chandra and Ammanabrolu, Prithviraj and Choi, Yejin and Ren, Xiang},
title = {{SWIFTSAGE: a generative agent with fast and slow thinking for complex interactive tasks}},
year = {2023},
publisher = {Curran Associates Inc.},
booktitle = {Proceedings of the 37th International Conference on Neural Information Processing Systems},
articleno = {1034},
numpages = {13},
location = {New Orleans, LA, USA},
series = {NIPS'23}
}

@misc{christakopoulou2024agentsthinkingfastslow,
      title={Agents Thinking Fast and Slow: A Talker-Reasoner Architecture}, 
      author={Konstantina Christakopoulou and Shibl Mourad and Maja Matarić},
      year={2024},
      eprint={2410.08328},
      archivePrefix={arXiv},
      primaryClass={cs.AI},
      url={https://arxiv.org/abs/2410.08328}, 
}

@article{zoller2021benchmarksurveyautomatedmachine,
      title={Benchmark and Survey of Automated Machine Learning Frameworks}, 
      author={Marc-André Zöller and Marco F. Huber},
      year={2021},
      journal={Journal of Artificial Intelligence Research (JAIR)},
      volume={70},
      pages={409--472},
      doi={10.1613/jair.1.11854}}

@misc{diazrodriguez2018dontforgetforgettingnew,
      title={Don't forget, there is more than forgetting: new metrics for Continual Learning}, 
      author={Natalia Díaz-Rodríguez and Vincenzo Lomonaco and David Filliat and Davide Maltoni},
      year={2018},
      eprint={1810.13166},
      archivePrefix={arXiv},
      primaryClass={cs.AI},
      url={https://arxiv.org/abs/1810.13166}, 
}

@article{Kirkpatrick_2017,
   title={Overcoming catastrophic forgetting in neural networks},
   volume={114},
   ISSN={1091-6490},
   DOI={10.1073/pnas.1611835114},
   number={13},
   journal={Proceedings of the National Academy of Sciences},
   author={Kirkpatrick, James and Pascanu, Razvan and Rabinowitz, Neil and Veness, Joel and Desjardins, Guillaume and Rusu, Andrei A. and Milan, Kieran and Quan, John and Ramalho, Tiago and Grabska-Barwinska, Agnieszka and Hassabis, Demis and Clopath, Claudia and Kumaran, Dharshan and Hadsell, Raia},
   year={2017},
   doi = {10.1073/pnas.1611835114},
   month=mar, 
   pages={3521–3526} 
}

@misc{Riemer18,
      title={Learning to Learn without Forgetting by Maximizing Transfer and Minimizing Interference}, 
      author={Matthew Riemer and Ignacio Cases and Robert Ajemian and Miao Liu and Irina Rish and Yuhai Tu and Gerald Tesauro},
      year={2019},
      eprint={1810.11910},
      archivePrefix={arXiv},
      primaryClass={cs.LG},
      url={https://arxiv.org/abs/1810.11910}, 
}

@inproceedings{Buzzega20,
author = {Buzzega, Pietro and Boschini, Matteo and Porrello, Angelo and Abati, Davide and Calderara, Simone},
title = {Dark experience for general continual learning: a strong, simple baseline},
year = {2020},
isbn = {9781713829546},
publisher = {Curran Associates Inc.},
booktitle = {Proceedings of the 34th International Conference on Neural Information Processing Systems},
articleno = {1335},
numpages = {11},
location = {Vancouver, BC, Canada},
series = {NIPS'20}
}

@article{self_adapative_systems,
author = {Gheibi, Omid and Weyns, Danny and Quin, Federico},
title = {Applying Machine Learning in Self-adaptive Systems: A Systematic Literature Review},
year = {2021},
publisher = {Association for Computing Machinery},
address = {New York, NY, USA},
volume = {15},
number = {3},
issn = {1556-4665},
doi = {10.1145/3469440},
journal = {ACM Trans. Auton. Adapt. Syst.},
month = aug,
articleno = {9},
numpages = {37}
}

@ARTICLE{chen2023aimaintenancerobustnessperspective,
  author={Chen, Pin-Yu and Das, Payel},
  journal={Computer}, 
  title={{AI Maintenance: A Robustness Perspective}}, 
  year={2023},
  volume={56},
  number={2},
  pages={48-56},
  doi={10.1109/MC.2022.3218005}}

@inproceedings{yao2023treethoughtsdeliberateproblem,
author = {Yao, Shunyu and Yu, Dian and Zhao, Jeffrey and Shafran, Izhak and Griffiths, Thomas L. and Cao, Yuan and Narasimhan, Karthik},
title = {Tree of thoughts: deliberate problem solving with large language models},
year = {2023},
publisher = {Curran Associates Inc.},
booktitle = {Proceedings of the 37th International Conference on Neural Information Processing Systems},
articleno = {517},
numpages = {14},
location = {New Orleans, LA, USA},
series = {NIPS'23}
}

@inproceedings{wang2024codeact,
author = {Wang, Xingyao and Chen, Yangyi and Yuan, Lifan and Zhang, Yizhe and Li, Yunzhu and Peng, Hao and Ji, Heng},
title = {Executable code actions elicit better LLM agents},
year = {2024},
publisher = {JMLR.org},
booktitle = {Proceedings of the 41st International Conference on Machine Learning},
articleno = {2054},
numpages = {25},
location = {Vienna, Austria},
series = {ICML'24}
}

@misc{kim2025scaling,
      title={Towards a Science of Scaling Agent Systems}, 
      author={Yubin Kim and Ken Gu and Chanwoo Park and Chunjong Park and Samuel Schmidgall and A. Ali Heydari and Yao Yan and Zhihan Zhang and Yuchen Zhuang and Mark Malhotra and Paul Pu Liang and Hae Won Park and Yuzhe Yang and Xuhai Xu and Yilun Du and Shwetak Patel and Tim Althoff and Daniel McDuff and Xin Liu},
      year={2025},
      eprint={2512.08296},
      archivePrefix={arXiv},
      primaryClass={cs.AI},
      url={https://arxiv.org/abs/2512.08296}, 
}

\end{document}